\documentclass[twoside]{article}
\usepackage{aistats2016}
\usepackage{times}
\usepackage{helvet}
\usepackage{courier}
\usepackage{times}
\usepackage{macros}
\usepackage{xr}
\usepackage{mathtools}

\begin{document}

%

%

\twocolumn[

\aistatstitle{Two Phase  $Q-$learning for Bidding-based Vehicle Sharing}

\aistatsauthor{ Yinlam Chow \And Jia Yuan Yu \And Marco Pavone }

\aistatsaddress{ Stanford University \And Concordia University \And Stanford University } ]

\begin{abstract}
We consider one-way vehicle sharing systems where customers can rent a car at one station and drop it off at another. The problem we address is to optimize the distribution of cars, and quality of service, by pricing rentals appropriately. We propose a bidding approach that is inspired from
  auctions and takes into account the significant uncertainty inherent in the problem data (e.g., pick-up and
  drop-off locations, time of requests, and duration of trips). Specifically, in contrast to
  current vehicle sharing systems, the operator does not set
  prices. Instead, customers submit bids and the operator decides whether to
  rent or not. The operator can even accept negative bids to motivate
  drivers to rebalance available cars to unpopular destinations within a city. We model
  the operator's sequential decision-making problem as a \emph{constrained
    Markov decision problem} (CMDP) and propose and rigorously analyze a novel two phase $Q$-learning algorithm for its solution. Numerical experiments are presented and discussed.
\end{abstract}

\section{Introduction}
\label{sec:introduction}
One-way vehicle sharing systems represent an increasingly popular mobility paradigm 
aimed at effectively utilizing usage of idle vehicles, reducing
demands for parking spaces, and possibly cutting  down excessive carbon footprints due to personal
transportation. One-way vehicle sharing systems (also referred to as mobility-on-demand --MOD-- systems) consist of a network
of parking stations and a fleet of vehicles. A customer arriving at a given station  can pick up a vehicle (if available) and drop it off at any other
station within the city. Existing vehicle sharing systems include Zipcar
\cite{keegan2009zipcar}, Car2Go \cite{schmauss2009car2go} and
Autoshare \cite{reynolds2001autoshare} for one-way car sharing, and
Velib \cite{nair2013large} and City-bike \cite{didonato2002city} for
one-way bike sharing. Figure \ref{fig:example_car_sharing} shows a
typical Toyota i-Road one-way vehicle sharing system 
\cite{kendall2013toyota}.
 \begin{figure}[!h]
   \begin{center}
     \includegraphics[width=\linewidth]{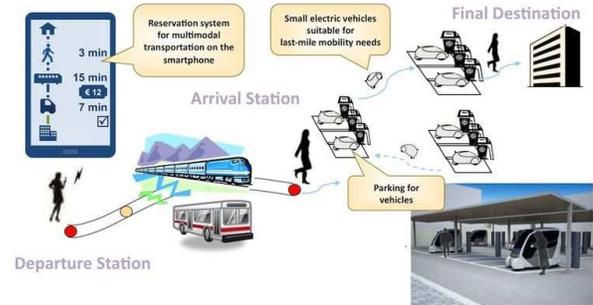}
   \end{center}
   \vspace{-0.35in}
   \caption{A Typical one-way vehicle sharing system that allows different
 pick-up and drop-off locations \cite{kendall2013toyota}.} \label{fig:example_car_sharing}
 \end{figure}


Despite the apparent advantages,  one-way vehicle sharing systems
present significant operational challenges. Due to the
asymmetry of travel patterns within a city, several  stations will eventually
experience imbalances of vehicle departures and customer
arrivals. Stations with low customer demands (e.g., in suburbs) will possess 
excessive unused vehicles and require a large number of parking spaces, while
stations with high demands (e.g., in the city center) will not be able to fulfill most
customers' requests during rush hours.


\emph{Literature Review}: In general, there are two main methods available in the literature to address demand-supply
imbalances in one-way vehicle sharing systems (however, if the vehicles can drive autonomously, additional rebalancing strategies are possible \cite{zhang2015control}).
A first class of methods is to hire crew drivers to periodically relocate vehicles among stations. From a theoretical standpoint, optimal rebalancing of drivers and vehicles has been analyzed in \cite{zhang2015queue}, under the framework of queueing networks.  In  \cite{barth1999simulation}, \cite{kek2006relocation},
\cite{shu2010bicycle}, the effectiveness of similar rebalancing strategies is numerically investigated  via discrete event
simulations. The work in \cite{nair2011fleet} considers a stochastic mixed-integer
programming   model where the objective is to minimize  
vehicle relocation cost subject to a probabilistic constraint on the service level. The  works in  \cite{Tal2013} and \cite{Mauro2014} consider a similar approach. While rebalancing can be quite effectively carried out by this way, 
these methods substantially increase sunk costs due to the large number of staff drivers that needs to be hired (a characterization of the number of drivers is provided in  \cite{zhang2015queue}), and may not scale well to large transportation networks.

Alternatively, demand-supply
imbalances can be addressed by imposing incentive pricing to vehicle
rentals. A typical incentive pricing mechanism is described  in
\cite{papanikolaou2011market} and portrayed in Figure
2. 
\begin{figure}[!h]
\begin{center}
\includegraphics[width=0.8\linewidth]{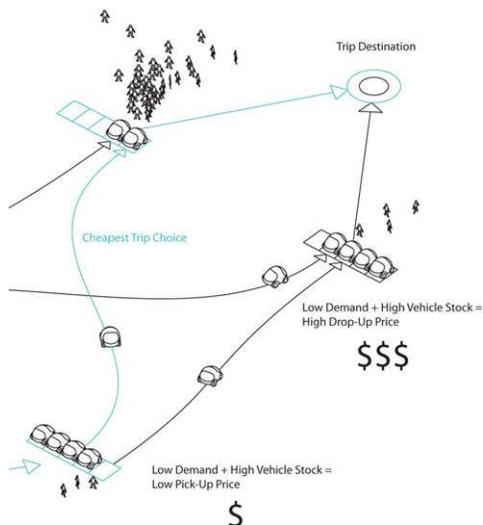}
\end{center}
\vspace{-0.35in}
\caption{Depiction of an incentive pricing mechanism where rental prices are adjusted  based on inventories and customers' demands \cite{papanikolaou2011market}.} 
\end{figure}
The strategy  is to adjust rental prices at each station  as a function of current
vehicle inventory and customers' requests. The work in \cite{uesugi2007optimization} proposes a method
to optimize vehicle assignments by trip splitting and trip
joining, and  \cite{mitchell2010reinventing} proposes
a dynamic pricing strategy that enables clients of a carpooling system to 
trade-off  convenience of a trip (i.e.,   duration) and cost.  Carpooling strategies, however, may not be a scalable
solution due to safety, convenience, and sociological reasons. Recently, the work in \cite{chow2015real} proposes a
bidding mechanism for vehicle rentals where at each station customers
place bids   and the
vehicle sharing company decides which bids to accept.  The
operator's sequential decision-making problem is posed as a  \emph{constrained Markov
  decision problem} (CMDP), which is solved exactly, or
approximately using an \emph{actor-critic} method.

A bidding strategy, such as the one proposed in \cite{chow2015real}, is attractive for several
reasons. First, accepted vehicle rental bids instantly reflect current
demands and supplies at different stations. Second, by providing
on-demand financial rewards for rebalancing vehicles, the rental
company can save overhead costs associated with  hiring crew drivers and renting extra
parking spaces. Third, such pricing mechanism promotes high  vehicle
utilizations by encouraging extra vehicle rentals to less popular
destinations and during non-rush hours.


\emph{Contributions}: The contribution of this paper is threefold. 
\begin{itemize}

\item Leveraging our recent findings  in \cite{chow2015real},  we propose a novel CMDP formulation for the problem of optimizing one-way vehicle sharing systems. The actions in the CMDP model represent vehicle rental decisions as a function of bids placed by customers. The objective  is to maximize total revenue from rental assignments, subject to a vehicle utilization constraint \footnote{The vehicle utilization constraint ensures that the assignment policy does  not excessively  favor short-term rentals. Accordingly, this constraint provides service guarantees for customers in need of long-term rentals. Further details about the practical relevance of this constraint can be found   in \cite{carsharing_report}.}. 
\item We derive a two-phase Bellman optimality condition for CMDPs. Such condition shows that CMDPs can be solved exactly by two-phase dynamic programming (DP).
\item We propose a novel sampling-based two-phase $Q-$learning algorithm for the solution of CMDPs, and show under mild assumptions convergence to an optimal solution.
\end{itemize}

This paper provides a first step toward designing market-based mechanisms for the operational optimization of  one-way vehicle sharing systems. We describe a wealth of open problems at the end of the paper. Furthermore, the results concerning CMDPs are of independent interest and applicable more broadly. Due to space limitations, in this paper we only include the statements of our  theoretical results. All proofs can be found in the supplementary material section. 


\section{Mathematical Model}\label{sec:veh_route_simple}

In this section we present a mathematical model for one-way vehicle sharing systems and then pose a CMDP decision-making problem for its optimization.

\subsection{General Model and Problem Data}\label{sec:input_1}

Assume the vehicle sharing company owns  $C$ vehicles, indexed  $1,\ldots,C$, that can be parked at $S$ stations, indexed $1,\ldots,S$. The company  only allows  passenger to rent for a maximum of $\overline{T}$ time slots; furthermore,  the maximum fare/reward for each rental period is $\overline{F}$. We consider a discrete time model $t=0,1,2,\ldots, T$, where $T$ is the time horizon. At time $t\geq 0$ and at each station $j\in\{1,\ldots,S\}$,  customers' destinations, rental durations, and proposed travel fares can be modeled via a multi-variate (three-dimensional) stationary probability distributions $\Phi_t^j$ with domain $\{1,\ldots,S\}\times (0,\overline{T}]\times[-\overline F,\overline F]$. 
The distribution $\Phi_t^j$ models the inherent stochastic, time-variant customer demand for car rentals.  A practical approach to estimate the demand distribution $\Phi^j_t$ at each station $j$ is to use data-driven methods, for example \cite{epanechnikov1969non}. 

At each time $t$ and station $j\in\{1,\ldots,S\}$, there are $M_{t}^j$ rental requests, where $M_{t}^j$ is modeled as a Poisson random variable with rate $\lambda_t^j$ (as usually done in the literature, see, for example, \cite{lim2007stochastic}). The associated  customers' destinations, rental durations, and proposed travel fares are i.i.d. samples drawn from the distribution $\Phi_t^j$. Such samples are collected in a random vector $\omega_t^j$, that is:
\[
\omega_t^j=((\mathbf{G}^{1,j}_t,\mathbf{T}^{1,j}_t,\mathbf{F}^{1,j}_t),\ldots,(\mathbf{G}^{M_{t}^j,j}_t,\mathbf{T}^{M_{t}^j,j}_t,\mathbf{F}^{M_{t}^j,j}_t)),
\]
where random variables $\mathbf{G}$, $\mathbf{T}$, and $\mathbf{F}$ denote, respectively, a customer's destination, rental  duration, and proposed fare (the meaning of the indices is clear from the definitions).

For $j\in\{1,\ldots,S\}$, denote by $\mathcal A^{j,j'}_{t}$ the number of customers that arrive at time $t$ at station $j$ and that  wish to travel to station $j'$. By definition,  $\mathcal A^{j,j'}_{t}=\sum_{m=1}^{M_t^j}\mathbf 1\{\mathbf{G}^{m,j}_t=j'\}$, where $1\{\cdot\}$ represents the indicator  function. 

This model captures both concepts of renting and rebalancing. Note that at station $j$, the random price offered by customer $m$, i.e., $\mathbf{F}^{m,j}_t$, can either be positive or negative. When this quantity is positive, it means that the customer is willing to pay $\mathbf{F}^{m,j}_t$ fare units for a vehicle, traveling to station $\mathbf{G}^{m,j}_t$  in $\mathbf{T}^{m,j}_t$ time units. If this quantity is negative, it means that the company is paying $\mathbf{F}^{m,j}_t$ fare units  to customer $m$ to travel to station $\mathbf{G}^{m,j}_t$ in $\mathbf{T}^{m,j}_t$ time units to fulfill rebalancing needs.


Note  that at each time instant and at each station, there can potentially be  more rental requests than vehicles available. The strategy is to   rank all incoming rental requests by a price-to-travel time function and assign vehicles according to such ranking. Specifically, for destination $j'\in\{1,\ldots,S\}$, we define the price-to-travel time function as
{\small
\[
f^{j'}_{\text{rank}}(\mathbf{G},\mathbf{T},\mathbf{F})\!=\!\left\{\!\!\begin{array}{cl}
\mathbf 1\{\mathbf{F}\geq 0\}\mathbf{F}/\mathbf{T}+\mathbf 1\{\mathbf{F}\leq 0\}\mathbf{F}\mathbf{T}&\text{if $\mathbf{G}= j'$}\\
-\infty&\text{otherwise}
\end{array}\right.
\]
}

By assigning vehicles according to the price-to-travel time function (in a descending order), one favors customers with high rental prices and short travel times (when renting), and drivers with low rewards and short travel times (when rebalancing). (If all vehicles move at similar speeds, one could equivalently consider distances instead of times.)

To explicitly model the instantaneous demand-supply imbalance throughout the transportation network, one can refine function $f^{j'}_{\text{rank}}(\mathbf{G},\mathbf{T},\mathbf{F})$ by including dependencies on the arrival station $j$ and time instant $t$. While this generalization is straightforward, we omit the details in the interest  of brevity.

\subsection{State Variables}\label{sec:state}
We now proceed to construct a CMDP model for a one-way vehicle sharing system. We consider the following state variables:
\vspace{-0.05in}
\begin{itemize}
\item For $t\geq 0$, $k_t =t \in\{0, \ldots,T - 1\}$ is  a counter state.
\item For $i\in\{1,\ldots,C\}$ and $t\geq 0$, $q^i_{t}\in \{1,\ldots,S\}$ is the destination station at time $t$ of the $i^{\text{th}}$ vehicle. Let $q_{t}:=(q^1_t,\ldots, q^C_t)$.
\vspace{-0.05in}
\item For $i\in\{1,\ldots,C\}$ and $t\geq 0$, $\tau^i_{t}\in\{0,1,2,\ldots,\overline{T}\}$ is the remaining travel time to reach the destination for the $i^{\text{th}}$ vehicle. Let $\tau_{t}:=(\tau^1_t,\ldots,\tau^C_t)$.
\end{itemize} 
Collectively, the state
space is defined  as $\mathbf X = \{0, \ldots, T -1\} \times \{1, \ldots, S\}^C \times
\{0, 1, 2, \ldots, \overline{T} \}^C$.
We let $x_0 = (0, q_0, \tau_0)$ denote the initial state, and $x_t = (k_t, q_t, \tau_t)$ the state at time $t$.
\subsection{Decision Variables}
At each time $t$, in order to maximize  expected revenue, the company makes decisions about renting vehicles to customers.    Specifically, at each time $t$, the decision variables are:
\begin{itemize}
\item For each station $j\in\{1,\ldots,S\}$, each vehicle $i\in\{1,\ldots,C\}$,  and $t\geq 0$, $u^{i,j}_t\in\{0,1\}$ is a binary decision variable that indicates whether vehicle $i$ is destined to station $j$ at time $t$. Let  $u_{t}:=(u^{1,1}_t,\ldots, u^{1,S}_t,\ldots,u^{C,1}_t\ldots,u^{C,S}_t)$.
\end{itemize}
For each station-destination pair $(j,j')\in\{1,\ldots,S\}^2$, we consider the following constraint to upper bound the number of vehicle dispatches at time $t\geq 0$:
\begin{equation}\label{eq:control_cons}
\begin{split}
&\sum_{i=1}^Cu^{i,j'}_{t}\mathbf{1}\{q^i_t=j,\,\tau^i_t=0\}\leq \mathcal{A}^{j,j'}_{t}\,\,\,\text{if $j\neq j'$},\\
&\sum_{i=1}^Cu^{i,j}_{t}\mathbf{1}\{q^i_t=j,\,\tau^i_t=0\}\leq \sum_{i=1}^C\mathbf{1}\{q^i_t=j,\,\tau^i_t=0\}.
\end{split}
\end{equation}
Intuitively, constraints \eqref{eq:control_cons} restrict the number of vehicle dispatches to be less than the number of customer requests. In the special case when $j=j'$, the upper bound is $\sum_{i=1}^C\mathbf{1}\{q^i_t=j,\,\tau^i_t=0\}$ instead of $\mathcal{A}^{j,j}_{t}$ because one needs to take into account the case when vehicle $i$ stays idle at station $j$.
Additionally, we consider  the following constraints to guarantee well-posedness  of the vehicles assignments:
\begin{equation}\label{eq:control_cons2}
\begin{split}
&u^{i,j}_{t}  =1,\,\,\forall i\in\{1,\ldots,C\},\,\,\text{if $\tau_t^i>0$ and $q_t^i=j$},\\
&\sum_{j=1}^Su^{i,j}_{t} =1,\,\,\forall i\in\{1,\ldots,C\}.
\end{split}
\end{equation}

Accordingly, the action space is $\mathbf U=\{0,1\}^{C\times S}$, and $u_{t}$ is the action taken at time $t$. Furthermore, define the set of admissible controls at state $x\in\mathbf X$ as $\mathbf{U}(x)\subseteq\mathbf{U}$, such that $\mathbf{U}(x)=\{u\in\mathbf{U} \text{ and it satisfies constraints \eqref{eq:control_cons} and \eqref{eq:control_cons2}}\}$.

\subsection{State Dynamics}\label{sec:input_2}
For  each time $t$ and station-destination pair $(j,j')\in\{1,\ldots,S\}^2$, 
 let 
\[
\mathcal S_t(j,j')=\{i\in\{1,\ldots,C\}:\tau^{i}_t=0,q_t^{i}=j,u^{i,j'}_{t}=1\},
\]
denote the set of vehicles at location $j$ which are allocated to destination $j'$. For $j\neq j'$, denote the rental durations and proposed fares of the $|\mathcal S_t(j,j')|$ highest ranked customers (according to the price-to-travel time function $f^{j'}_{\text{rank}}(\cdot)$) as $\{(\mathbf{T}_{t,j,j'}^{m},\mathbf{F}_{t,j,j'}^{m})\}_{m=1}^{|\mathcal S_t(j,j')|}$.

%
The case $j = j'$ is special,  as the quantity $|\mathcal S_t(j,j)|$ comprises vehicles that (1) will eventually return to station $j$, and (2) stay idle at station $j$. Notice that the definition of $ u_t^{i,j}$ does not allow one to distinguish between these two cases. However, in this case, it is not profitable  to rent vehicles to customers submitting negative bids. Hence one concludes  that  the number of vehicles that will eventually return to station $j$ should be equal to the number of customers that want to return to station $j$ and submit a positive bid. We denote such number as $V_t^{j,\text{in transit}}$. We similarly define the vector  $\{(\mathbf{T}_{t,j,j'}^{m},\mathbf{F}_{t,j,j'}^{m})\}_{m=1}^{|\mathcal S_t(j,j')|}$, with the understanding that the first  $V_t^{j,\text{in transit}}$ elements correspond to the customers that submit positive bids, while each remaining element is set to $(0,0)$.
%
%
%

The state dynamics are then given as follows:
\begin{itemize}
\item For each vehicle $i\in\{1,\ldots,C\}$ such that $\tau^i_t>0$,
\[
(k_{t+1}, q^i_{t+1},\tau^i_{t+1})=(k_t + 1, q_t^i,\max(\tau_t^i-1,0)).
\]
\item For all other vehicles, for each station-destination pair $(j,j')\in\{1,\ldots,S\}^2$, and each vehicle $i\in\mathcal S_t(j,j')$, uniformly randomly allocate $\{\mathbf{T}^{1}_{t,j,j'},\ldots,\mathbf{T}^{|\mathcal S_t(j,j')|}_{t,j,j'}\}$ to the available vehicles, i.e., 
\[
(k_{t+1}, q^i_{t+1},\tau^i_{t+1})=\left(k_t + 1,  j',\mathbf{T}_{t,j,j'}^{m}\right),
\] 
for all $m\in\{1,\ldots,|\mathcal S_t(j,j')|\}$.
\end{itemize}

As the problem is high-dimensional even for a moderate number of vehicles and stations, and since the state update equations depend in a rather involved fashion on the information vector $\omega_t^j$, the explicit derivation of the state transition probabilities (denote by $\mathbb P$) is impractical. This motivates the $Q$-learning approach proposed in this paper.
%
%

\subsection{Revenue and Constraint Cost Functions}\label{sec:revenue}
Define the   {\em immediate} reward function $\mathbf{R}: \mathbf{X}\times \mathbf{U}\rightarrow \mathbb{R}$ as
\[
\mathbf{R}(x,u)\!= 
\mathbb E\left[\sum_{j=1}^S\sum_{j'=1}^S\sum_{m=1}^{|\mathcal S_t(j,j')|}\mathbf{F}_{t,j,j'}^{m}\right].
\]
The total revenue is then given by
\[
\sum_{t=0}^{T}\mathbb E\left[\mathbf{R}(x_{t},u_{t})\right].
\] 

From a pure profit maximization standpoint, an operator should favor short term rental assignments, in order to minimize the opportunity cost of rejecting future customers that might potentially be more profitable. While this strategy optimizes the long-term revenue, it is not fair toward customers that require extended rental periods. To balance total profit with customers' satisfaction, we impose an additional  constraint that lower bounds vehicle utilization. Specifically,  the  constraint cost function  $\mathbf{D}:\mathbf{X}\times\mathbf U\rightarrow \mathbb{R}$ captures vehicle utilization at each time instant according to the definition 
\[
\mathbf{D}(x,u)=  \mathbf  d 
-\sum_{i=1}^C \frac{\tau^i}{TC},
\] 
where  $\mathbf d$ represents a threshold for the average utilization rate specified by the system operator.
The vehicle utilization constraint is then given by
\[
\sum_{t=0}^{T}\mathbb E\left [\mathbf{D}(x_t,u_t)\right] \leq 0.
\]

The objective   is then to maximize expected revenue while satisfying the vehicle utilization constraint. 

\subsection{CMDP Formulation}\label{sec:problem}
Equipped with the state space $\mathbf{X}$, control space $\mathbf{U}$, immediate reward $\mathbf{R}$, transition probability $\mathbb{P}$ (implicitly defined above), initial state $x_0$, and immediate constraint cost function $\mathbf D$, we pose the problem of controlling a vehicle sharing system as a CMDP. Specifically, note that the problem can be formulated on an infinite horizon by interpreting the set of states $\mathcal X=\{x_t \in \mathbf X:k_t=T\}$ (indeed containing a single state) as \emph{absorbing}, i.e., once $k_t = T$, the system enters a termination state, call it $\texttt{END}$, and stays there with \emph{zero} reward and constraint cost. From this perspective, $T$ represents a \emph{deterministic} stopping time, and the upper limits in the summations for the reward and cost functions can be replaced with $\infty$. We also define $\mathbf X'=\mathbf X\setminus \mathcal X$ as the set of \emph{transient states}.  Let $\Pi_M$ be the set of closed-loop, Markovian stationary policies $\mu: \mathbf{X}\rightarrow \mathbb P(\mathbf{U})$. A policy $\mu\in \Pi_M$ induces a stationary mass distribution over the realizations
of the stochastic process $(x_t,u_t)$.
It is well known that for CMDPs there is no loss of optimality in restricting the attention to policies in $\Pi_M$ (instead, e.g., of also considering history-dependent policies).

Accordingly, in this paper we wish to solve the CMDP problem: 
\begin{quote}{\bf Problem $\mathcal{OPT}$} -- Solve
\begin{align}
\text{maximize}_{\mu\in\Pi_M}  & &\quad & \mathbb E\left[\sum_{t=0}^{\infty}\mathbf{R}(x_{t},u_{t})\mid x_{0}\;{u}_{t}\sim\mu\right]\nonumber\\
\text{subject to}  & & \quad & \mathbb E\left[ \sum_{t=0}^{\infty} \mathbf{D}(x_{t},u_{t})\mid x_{0}\;{u}_{t}\sim\mu\right]\leq 0\nonumber.
\end{align}
\end{quote}
Since, in our problem,  the state and action spaces are exponentially large (in terms of the vehicle number $C$ and station number $S$) and  the explicit derivation  of the state transition probability is intractable, exact solution methods for CMDPs (e.g., \cite{altman1999constrained}) are not applicable. 

On the other hand, even if one computes an optimal randomized policy for Problem $\mathcal{OPT}$, executing such policy on a vehicle sharing platform with simple system architecture may lead to problems in vehicle mis-coordination \cite{paruchuri2004towards}. This motivates us to restrict the structure of admissible policies in problem $\mathcal{OPT}$ to Markovian, stationary and deterministic. In the next section we will introduce a two-phase Bellman optimality condition for problem $\mathcal{OPT}$, which serves as the theoretical underpinning for the design of an asymptotically optimal two-phase $Q-$learning algorithm. 

\section{Two-Phase Dynamic Programming Algorithm}
In this section, by leveraging the results in \cite{gabor1998multi}, we present a two-phase DP algorithm for problem $\mathcal{OPT}$. As we shall see, the first step is to compute a $Q$-function that allows one to \emph{refine} the set of feasible control actions (essentially, we retain only those actions that can guarantee the fulfillment of the constraint in problem $\mathcal{OPT}$). We then define a Bellman operator (restricted to the refined set of control actions) that allows the computation of an  optimal policy. The DP algorithm presented in this section provides the conceptual basis for the two-phase $Q-$learning algorithm presented in the next section. 

\subsection{Phase 1: Finding the Feasible Set}
In this section we  characterize the feasible set for problem $\mathcal{OPT}$ using the set of optimal policies from an auxiliary MDP problem, defined as follows:
\begin{quote}{\bf Problem $\mathcal{FS}$} -- Solve
\begin{equation}\label{eq:opt_new}
\min_{\mu \in \Pi_{MD}} \max\left\{0,\mathbb E\left[\sum_{t=0}^{\infty}\mathbf{D}(x_{t},u_t)\mid x_0,\mu\right]\right\},
\end{equation}
where $\Pi_{MD}$ is the set of closed-loop, Markovian, stationary, and deterministic policies.
\end{quote}
The following (trivial) result related minimizer for problem $\mathcal{FS}$ with feasible solutions for problem $\mathcal{OPT}$.

\begin{lemma}\label{tech_lem1}
Let $\mu:\mathbf X\rightarrow\mathbf U$ be a Markovian stationary deterministic policy that minimizes problem $\mathcal{FS}$, that is 
\[
\mu \in \arg\min_{\mu} \max\left\{0,\mathbb E\left[\sum_{t=0}^{\infty}\mathbf{D}(x_{t},u_t)\mid x_0,\mu\right]\right\}.
\]
Then the solution cost is equal to zero if and only if 
\[
\mathbb  E\left[\sum_{t=0}^{\infty}\mathbf{D}(x_{t},u_t)\mid x_0,\mu\right]\leq 0. 
\]

\end{lemma}
Equipped with the above result, one can immediately deduce the following:
\begin{itemize}
\item If the solution cost to problem $\mathcal{FS}$ is strictly larger than zero, i.e.,
\[
\min_{\mu\in\Pi_{MD}} \max\left\{0,\mathbb E\left[\sum_{t=0}^{\infty}\mathbf{D}(x_{t},u_t)\mid x_0,\mu\right]\right\}>0,
\]
then problem $\mathcal{OPT}$ is infeasible.
\item Otherwise, the feasible set of Markovian stationary deterministic policies is given by
\[
\Pi_{\text{FS}}=\arg\min_{\mu\in\Pi_{MD}} \max\left\{0,\mathbb E\left[\sum_{t=0}^{\infty}\mathbf{D}(x_{t},u_t)\mid x_0,\mu\right]\right\}.
\]
\end{itemize}
In order to characterize the feasible set $\Pi_{\text{FS}}$, we derive a Bellman optimality condition for problem $\mathcal{FS}$ and demonstrate how $\Pi_{\text{FS}}$ can be computed via DP. 

Before getting into the main result, we define the Bellman operator for problem $\mathcal{FS}$ as follows:
\[
\begin{small}
\mathbf T[V](x)\!:=\!\! \min_{u\in\mathbf U(x)}\!\! \max\!\left\{\!\mathcal B(x),\mathbf{D}(x,u)+\!\!\!\!\sum_{x^\prime\in \mathbf X'}\!\!\mathbb{P}(x^\prime|x,u)V(x')\!\right\},
\end{small}
\]
where $\mathcal B(x)$ is the indicator function
\[
\mathcal B(x)=\left\{\begin{array}{cl}
0&\text{if $ x\in\mathcal X$}\\
-\infty&\text{otherwise}
\end{array}\right..
\]
With such definition of Bellman operator $\mathbf T$, we will later see that the fixed point solution of $\mathbf T[V](x)=V(x)$, $\forall x\in{\mathbf{X}'}$ is equal to the solution of problem $\mathcal{FS}$, given by \eqref{eq:opt_new}. 

For any bounded initial value function estimate $V_0:{\mathbf{X}}\rightarrow\mathbb R$ with $V_0(x)=0$ for $x\in\mathcal X$,  we define
the  value function sequence 
\begin{equation}\label{eq:VI_FEA}
V_{k+1}(x)=\mathbf T[V_k](x),\,\,\forall x\in\mathbf X',\,\,k\in\{0,1,\ldots\}.
\end{equation}
The following theorem shows that this sequence of value function estimates converges to the solution of problem $\mathcal{FS}$, which is also the unique fixed point of $\mathbf T[V](x)=V(x)$, $\forall x\in\mathbf X'$. 
\begin{theorem}[Bellman Optimality for $\mathcal{FS}$]\label{thm:opt}
For any initial value function estimate $V_0:{\mathbf{X}}\rightarrow\mathbb R$ where $V_0(x)=0$ at $x\in\mathcal X$, there exists a limit function $V^*$ such that 
\begin{equation}\label{eq:equiv}
\begin{split}
V^*(x_0)=&\lim_{N\rightarrow\infty} \mathbf T^N[V_{0}](x_0)\\
=&\min_{\mu\in\Pi_M} \!\max\left\{\!0,\mathbb E\left[\sum_{t=0}^{\infty}\mathbf{D}(x_{t},u_t)\!\mid\! x_0,\mu\right]\!\right\}.
\end{split}
\end{equation}
Furthermore, $V^*$ is a unique solution to the fixed point equation: $\mathbf T[V](x)=V(x)$, $\forall x\in{\mathbf{X}'}$.
\end{theorem}
By running the value iteration algorithm in \eqref{eq:VI_FEA}, one obtains the optimal value function for problem $\mathcal{FS}$. If $V^*(x_0) = 0$, then, by Lemma \ref{tech_lem1}, every  feasible policy for problem $\mathcal{OPT}$, denoted by $\mu_{\text{FS}}$, can be obtained as 
{\small
\begin{equation*}
\mu_{\text{FS}}(x)\!\in\!\arg\!\!\!\min_{u\in\mathbf U(x)}\!\!\!\max\left\{\!\mathcal B(x),\mathbf{D}(x,u)\!+\!\!\!\sum_{x^\prime\in \mathbf X'}\!\!\mathbb{P}(x^\prime|x,u)V^*(x')\!\right\}.
\end{equation*}
}

However, since the number of feasible policies is exponential in the  size of state and action spaces, their exhaustive enumeration is intractable, and the above result is useful only form a conceptual standpoint. To address this problem, we consider a refined notion of feasible control actions in terms of  optimal state-action value functions ($Q-$functions), which provides the basis for the two-phase DP algorithm. Specifically, a $Q-$function is defined as:
{\small
\[
Q^*(x,u):= \max\!\left\{\!\mathcal B(x),\mathbf{D}(x,u)+\!\!\!\!\sum_{x^\prime\in \mathbf X'}\mathbb{P}(x^\prime|x,u)V^*(x')\right\},
\]
where $V^*(x)= \min_{u\in\mathbf U(x)}Q^*(x,u)$ for $x\in\mathbf X'$.
By defining the state-action Bellman operator 
\[
\mathbf F[Q](x,u)\!=\! \max\left\{\!\mathcal B(x), \!\mathbf{D}(x,u)\!+\!\!\!\sum_{x^\prime\in \mathbf X'}\!\!\mathbb{P}(x^\prime|x,u)\!\! \min_{u'\in\mathbf U(x')}Q(x',u')\!\right\},
\]
equivalently $Q^*$ is a unique fixed point solution to $\mathbf F[Q](x,u)=Q(x,u)$ for any $u\in\mathbf U(x)$, $x\in\mathbf X'$. Note the interpretation of $Q^*(x,u)$ from these equations: it is the constraint cost of starting at state $x\in\mathbf X'$, using control action $u$ in the first stage, and using an optimal policy of problem $\mathcal{FS}$ thereafter. 

We are now in a position to define a \emph{refined} set of feasible control actions, denoted by $\mathbf U_{\text{FS}}(Q^*,x)$, whereby we retain only those actions that can guarantee the fulfillment of the constraint in problem $\mathcal{OPT}$. Specifically, for any state $x\in\mathbf X'$, we define:
{\small
\[
\begin{split}
\mathbf U_{\text{FS}}&(Q^*,x):= \{u\in\mathbf U(x):\exists  \mu_{\text{FS}}\in\Pi_{\text{FS}} \text{ such that } u=\mu_{\text{FS}}(x)\}\\
=&\bigg\{\!u\in\mathbf U(x)\!:\!u\in \!\!\arg \, \min_{u'\in\mathbf U(x)}\!\!Q^*(x,u') \text{ and } Q^*(x,u)\!=\! 0\!\bigg\}.
\end{split}
\]
}


\subsection{Phase 2: Constrained Optimization}
From the definition of the set of feasible policies $\Pi_{\text{FS}}$, one can reformulate  problem $\mathcal{OPT}$ as:
\begin{equation}\label{eq:cons_obj}
\max_{\mu\in\Pi_{\text{FS}}}\mathbb E\left[\sum_{t=0}^T{\mathbf{R}}(x_{t},u_{t})\mid x_0,\mu\right].
\end{equation}
This problem can be solved via value iteration by defining  the Bellman operator (with respect to the \emph{refined} set of control actions $\mathbf U_{\text{FS}}(Q^*,x)$):
{\small
\[
\mathbf T_R[W](x):=\max_{u\in\mathbf U_{\text{FS}}(Q^*,x)}\Bigg\{{\mathbf{R}}(x,u)+\sum_{x^\prime\in \mathbf X'}\mathbb{P}(x^\prime|x,u)W\left(x^\prime\right)\!\!\Bigg\}.
\]
}

The following theorem shows there exists a unique fixed point solution to $\mathbf T_R[W](x)=W(x)$ and such solution corresponds to the value function for the problem in \eqref{eq:cons_obj} (and, hence, problem $\mathcal {OPT}$).

\begin{theorem}[Bellman Optimality for $\mathcal{OPT}$]\label{thm:opt_2}
For any initial value function estimate $W_0:{\mathbf{X}}\rightarrow\mathbb R$ such that $W_0(x)=0$ for any $x\in\mathcal X$, there exists a limit function $W^*$ such that 
\[
\begin{split}
 W^*(x_0)=&\lim_{N\rightarrow\infty} (\mathbf T_R)^N[W_{0}](x_0)\\
 =&\max_{\mu\in\Pi_{\text{FS}}}\mathbb E\left[\sum_{t=0}^{\infty}{\mathbf{R}}(x_{t},u_{t})\mid x_0,\mu\right].
 \end{split}
\] 
Furthermore, $W^*$ is a unique solution to the fixed point equation: $\mathbf T_R[W](x)=W(x)$, for any $ x\in{\mathbf{X}'}$. 
\end{theorem}

Therefore, for any bounded initial value function estimate $W_0:{\mathbf{X}}\rightarrow\mathbb R$ such that $W_0(x)=0$ at $x\in\mathcal X$, the value function estimate sequence 
\begin{equation}\label{eq:VI_OPT}
W_{k+1}(x)=\mathbf T_R[W_k](x),\,\,\forall x\in\mathbf X',\,\,k\in\{0,1,\ldots,\},
\end{equation}
converges to the value function for problem $\mathcal{OPT}$. 

Finally, we define the $Q-$function for problem \eqref{eq:cons_obj} (denoted as $H$ to distinguish it from the $Q-$function for problem $\mathcal FS$):
\[
H^*(x,u) := \mathbf{R}(x,u)+\!\sum_{x^\prime\in \mathbf X'}\mathbb{P}(x^\prime|x,u) \, W^*(x'),
\]
where $W^*(x)$ is the value function for problem \eqref{eq:cons_obj}. By defining the state-action Bellman operator
{\small
\[
\mathbf F_R[H](x,u)\!=\! \mathbf{R}(x,u)+\!\sum_{x^\prime\in \mathbf X'}\mathbb{P}(x^\prime|x,u)\!\!\!\min_{u'\in\mathbf U_{\text{FS}}(Q^*,x')}\!\!\!H\left(x^\prime,u^\prime\right),
\]
}
one can show, similarly as before, that $H^*$ is a unique fixed point solution of  $\mathbf F_R[H](x,u)=H(x,u)$ for $u\in\mathbf U_{\text{FS}}(Q^*,x)$, $x\in\mathbf X'$. 

The above two-phase Bellman optimality condition immediately leads to a two-phase DP algorithm for the solution of problem $\mathcal{OPT}$ . However, such DP algorithm presents two main implementation challenges. First,  the algorithm is not applicable  to the vehicle sharing problem considered in this paper since the state transition probabilities are not available explicitly (as discussed above). Second, when the size of the state and action spaces are large, updating the value iteration estimates is computationally intractable. 

To address these computation challenges, in the next  section we present a sampling-based two-phase $Q-$learning algorithm that approximates the solution to problem $\mathcal{OPT}$. Similar to the two-phase DP algorithm, in the first phase one updates the $Q-$function estimates for problem $\mathcal{FS}$ ``by sampling" the vehicle sharing model. Then, in a second phase, one updates the $Q-$function estimates for  problem $\mathcal{OPT}$ (recall that such $Q-$functions are referred to as functions $H$).

\section{Two Phase $Q-$learning}
In this section we present both synchronous and asynchronous versions of two-phase $Q-$learning to solve problem $\mathcal{OPT}$. In the synchronous version, the $Q-$function estimates of all state-action pairs are updated at each step. In contrast, in the asynchronous version, only the $Q-$function estimate of a sampled state-action pair is updated. Under mild assumptions, we show that both algorithms are asymptotically optimal. While the convergence rate of synchronous $Q-$learning is higher \cite{kearns1999finite}, asynchronous $Q-$learning is more computationally efficient. 

\subsection{Synchronous Two Phase $Q-$learning}

Suppose $Q_0(x,u)$  is an initial $Q-$function estimate such that $Q_0(x,u)=0$ for any $x\in\mathcal X$. At iteration $k\in\{0,1,\ldots\}$, the synchronous two-phase $Q-$learning algorithm samples $N$ states $(x^{\prime,1},\ldots,x^{\prime,N})$ and updates the $Q-$function estimates for each state-action pair $(x,u)\in\mathbf X'\times\mathbf U$ as follows:
{\small
\begin{alignat}{2}
& Q_{k+1}(x,u)= Q_{k}(x,u)+\zeta_{2,k}(x,u)\cdot\bigg(  \max\bigg\{\mathcal B(x),\mathbf{D}(x,u)+\nonumber\\
 & \frac{1}{N}\sum_{m=1}^{N}\min_{u^{\prime,m}\in\mathbf{U}(x^{\prime,m})} Q_{k}(x^{\prime,m},u^{\prime,m})\bigg\}-  Q_{k}(x,u)\bigg),\label{eq:update_Q_1}\\
 &H_{k+1}(x,u)=H_{k}(x,u)+\zeta_{1,k}(x,u)\cdot \bigg( \mathbf{R}(x,u)+ \nonumber\\
&\frac{1}{N}\!\sum_{m=1}^{N}\max_{u^{\prime,m}\in \mathbf U_{\text{FS}}(Q_k,x^{\prime,m})}\!\!H_{k}(x^{\prime,m},u^{\prime,m}) \!-\! H_{k}(x,u)\!\bigg).\label{eq:update_Q_2}
\end{alignat}
}

The step size pair $(\zeta_{1,k}(x,u),\zeta_{2,k}(x,u))$ follows the update rule
\begin{equation}
\begin{split}
&\sum_k \zeta_{1,k}(x,u) = \sum_k \zeta_{2,k}(x,u) =\infty, \\
&\sum_k \zeta_{1,k}^2(x,u)<\infty,\;\sum_k \zeta_{2,k}^2(x,u)<\infty, \\
&\zeta_{1,k}(x,u) = o\big(\zeta_{2,k}(x,u)\big).\label{eq:step_incre}
\end{split}
\end{equation}

The last equation implies that the $Q-$function update $Q_{k}(x,u)$ is on the fast time scale, and the $Q-$function update $H_{k}(x,u)$ is on the slow time scale. 
Notice that in the sampling approach, the state trajectory will enter the absorbing set $\mathcal X$ in $T$ steps. While the convergence of $Q-$learning is a polynomial of $|\mathbf X|$ and $|\mathbf U|$ (see the finite sample analysis in Theorem 1 of \cite{kearns1999finite}), in order to get an accurate estimate of the $Q-$function one needs more state-action samples from the transient state space. However once the state trajectory enters $\mathcal X$, it will never visit the transient state space $\mathbf X'$ again. To collect more samples from the transient state space, similar to the approaches adopted by sampling based methods in \cite{yu2006least}, \cite{egloff2005monte}, here we reset the state to its initial condition immediately after it enters the absorbing set.
The convergence result for the synchronous two-phase $Q-$learning algorithm is given in the following theorem.
\begin{theorem}[Convergence of Synchronous $Q$-learning]\label{thm:syn_Q}
Suppose the step-sizes $(\zeta_{1,k}(x,u),\zeta_{2,k}(x,u))$ follow the update rule in \eqref{eq:step_incre}. Then the sequence of $Q-$function estimates computed via synchronous two-phase $Q-$learning converges to the optimal $Q-$function pair $(Q^*(x,u),H^*(x,u))$  component-wise with probability $1$.
\end{theorem}

After both $Q-$functions converge, a near-optimal policy can be computed as
\begin{equation}\label{eq:pol_opt}
\widetilde\mu^*(x) \in\arg\min_{u\in\mathbf U_{\text{FS}}(Q_{\bar k},x)} H_{\bar k}(x,u),\,\,\forall x\in\mathbf X',
\end{equation}
where $\bar k$ is the iteration index when the leaning is stopped.

\subsection{Asynchronous Two-Phase $Q-$learning}
Suppose $Q_0(x,u)$ is an initial $Q-$function estimate such that $Q_0(x,u)=0$ for any $x\in\mathcal X$. At iteration $k\in\{0,1,\ldots\}$ and state $x_k \in\mathbf X$, the asynchronous two-phase $Q-$learning algorithm (1) generates a control action
\begin{equation}\label{eq:control_opt}
u_k \in\arg\min_{u\in\mathbf U_{\text{FS}}(Q_k,x_k)} H_{k}(x_k,u),
\end{equation}
(2) samples $N$ states $(x^{\prime,1},\ldots,x^{\prime,N})$, and (3) updates the $Q-$function estimates as follows:
\begin{itemize}
\item for $x=x_k$ and $u=u_k$, $Q-$function estimates are updated according to equations \eqref{eq:update_Q_1} and \eqref{eq:update_Q_2},
\item otherwise, the $Q-$function estimates are equal to their previous values, i.e.,
\[
Q_{k+1}(x,u)=Q_{k}(x,u),\,\,H_{k+1}(x,u)=H_{k}(x,u).
\]
\end{itemize}

The convergence result for the asynchronous two-phase $Q-$learning algorithm is given in the following theorem.

\begin{theorem}[Convergence of Asynchronous $Q$-learning]\label{thm:asy_Q}
Suppose the step-sizes $(\zeta_{1,k}(x,u),\zeta_{2,k}(x,u))$ follow the update rule in \eqref{eq:step_incre}. Also suppose each state action pair $(x,u)\in\mathbf X\times\mathbf U$ is visited infinitely often. Then, the sequence of $Q-$function estimates computed via asynchronous two-phase $Q-$learning converges to the optimal $Q-$function pair $(Q^*(x,u),H^*(x,u))$ with probability $1$.
\end{theorem}

Note that the convergence result relies on the assumption that each state-action pair $(x,u)\in\mathbf X\times\mathbf U$ is visited infinitely often. While this is a standard assumption in the $Q-$learning literature \cite{bertsekas1995neuro}, by following analogous arguments as in \cite{strehl2006pac}, the above result can proven under milder assumptions by using PAC analysis. As for synchronous two-phase $Q-$learning, a near optimal policy can be computed by \eqref{eq:pol_opt} after the $Q-$functions converge.

In the next section, we perform numerical experiments to compare the proposed two-phase $Q-$learning algorithms to a number of alternative approaches. In particular, we consider a Lagrangian relaxation method \cite{altman1999constrained}, whereby one transforms problem $\mathcal{OPT}$ into a min-max MDP and solve for the optimal saddle point. However, finding the optimal Lagrange multiplier is a challenging problem. While multi-scale stochastic approximation algorithms such as actor-critic \cite{borkar2005actor} are available for optimizing both the Lagrange multiplier and policy online, in order to update the Lagrange multiplier, one requires sequential gradient approximations. This makes the convergence of these algorithms very sensitive to the multiple step-sizes and thus non-robust to large scale problems. Further numerical insights are provided below.

%
%

\subsection{Numerical Results}
Consider a small vehicle sharing system that consists of $15$ vehicles ($C=15$), $5$ stations ($S=5$), and a horizon of $6$ hours ($T=6$). In problem $\mathcal{OPT}$, one aims to find an optimal assignment strategy that maximizes the total revenue subject to the vehicle utilization constraint. The constraint threshold $\mathbf d$ is set equal to $38\times T$ to ensure that the average utilization time of each vehicle is at least $2.5$ hours. In our comparative study, we consider (1) a $Q-$learning  algorithm \cite{kearns1999finite} that  maximizes  total revenue and does not take into account the utilization constraint; (2) a penalized $Q-$learning algorithm, which maximizes a combined utility function of revenue and constraint violation penalty; (3) a $Q-$learning algorithm with Lagrangian update \cite{borkar2005actor}, which approximately solves problem $\mathcal{OPT}$ using an actor-critic method; and (4) the proposed two-phase $Q-$learning algorithm. Performance of these algorithms is evaluated via $1000$ Monte Carlo trials, for which the corresponding empirical rewards and constraint costs are shown in Figures \ref{fig:reward} and \ref{fig:constraint} respectively. The policy computed via $Q-$learning returns the highest total revenue, but the average vehicle utilization time is only $1.7$ hours.  On the other hand, the computed  policies from two-phase $Q-$learning and $Q-$learning with Lagrangian update decrease total revenue by $20\%$ but guarantee that average vehicle utilization time is over $2.5$ hours. Also, one can note that the proposed two-phase $Q-$learning algorithm converges faster than the $Q-$learning algorithm with Lagrangian update. Finally, the policy from penalized $Q-$learning\footnote{Here we perform grid search on the penalty parameter in order to maximize the total revenue while satisfying the vehicle utilization constraint.} has the highest average vehicle utilization time ($2.73$ hours) but lowest total revenue (with a $28\%$ gap). As a further comparison, a greedy policy that assigns as many rentals as possible provides a total revenue as low as $33.05$.

 \begin{figure}[!h]
   \begin{center}
     \includegraphics[width=\linewidth]{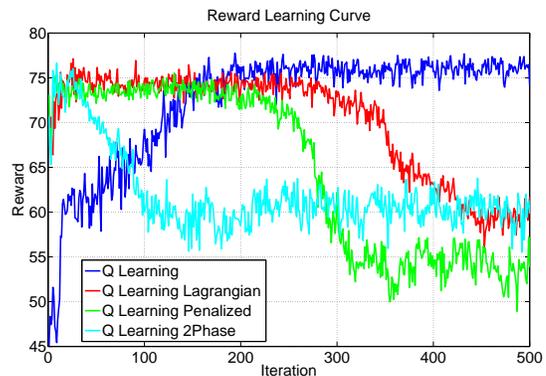}
   \end{center}
   \vspace{-0.3in}
   \caption{Reward Curve for Various Assignment Methods.} \label{fig:reward}
 \end{figure}
 \begin{figure}[!h]
   \begin{center}
     \includegraphics[width=\linewidth]{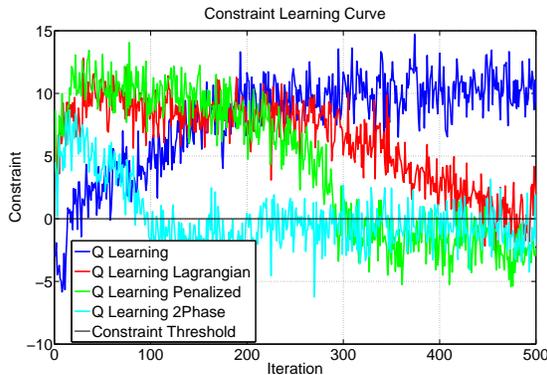}
   \end{center}
   \vspace{-0.3in}
   \caption{Constraint Curve for Various Assignment Methods.} \label{fig:constraint}
 \end{figure}
 
In each iteration of the two-phase $Q-$learning algorithm, the assignment problem in \eqref{eq:control_opt} is cast as a \emph{bilinear integer linear programming} (BILP) problem. Although BILP problems are NP-hard in general, with readily available optimization packages such as CPLEX \cite{cplex2009v12}, our algorithm is capable of solving medium-scale problems with up to $50$ vehicles, $20$ stations, and a horizon of $12$ hours. We believe there is still ample room for improvement by leveraging parallelization and characterizing the $Q-$functions with function approximations. 

\section{Conclusion}

In this paper, we propose a novel CMDP model for  one-way vehicle sharing systems whereby  the real-time rental assignment of vehicles relies on an auction-style bidding paradigm. We rigorously derive a two-phase Bellman optimality condition for the CMDP and show that this problem can be solved (conceptually) using two-phase dynamic programming. Building upon this result, we propose a practical sampling-based two-phase $Q-$learning algorithm  and show that the solution converges asymptotically to the value function of the CMDP.

Future work includes: {\bf 1)} Providing a convergence rate for our two-phase $Q-$learning algorithm; {\bf 2)} Generalizing the proposed bidding mechanism by using market-design mechanisms \cite{krishna2009auction} and game theory \cite{nisan2007algorithmic}; and {\bf 3)} Evaluating our algorithm on a large-scale vehicle sharing system. 

\section{Acknowledgements}
 Y.-L. Chow and M. Pavone are partially supported by Croucher Foundation Doctoral Scholarship, the Office of Naval Research, Science of Autonomy Program, under Contract N00014-15-1-2673, and by the National Science Foundation under CAREER Award CMMI-1454737. 

\bibliographystyle{plain} 
\bibliography{bidding}

\onecolumn
\appendix
\section{Appendix: Technical Proofs}
\subsection{Proof of Lemma \ref{tech_lem1}}
First notice that 
\[
\max\left\{0,\mathbb E\left[\sum_{t=0}^{\infty}\mathbf{D}(x_{t},u_t)\mid x_0,\mu\right]\right\}\geq 0.
\]
Thus for any minimizer $\mu^*$ of problem $\mathcal{FEA}$ such that the solution is $0$, it directly implies that 
\[
\mathbb E\left[\sum_{t=0}^{\infty}\mathbf{D}(x_{t},u_t)\mid x_0,\mu^*\right]\leq 0,
\] 
i.e., $\mu^*$ is a feasible policy of problem $\mathcal{OPT}$.

On the other hand, suppose a control policy $\mu$ is feasible to problem $\mathcal{OPT}$, i.e., 
\[
\mathbb E\left[\sum_{t=0}^{\infty}\mathbf{D}(x_{t},u_t)\mid x_0,\mu\right]\leq 0.
\]
This implies that 
\[
\max\left\{0,\mathbb E\left[\sum_{t=0}^{\infty}\mathbf{D}(x_{t},u_t)\mid x_0,\mu\right]\right\}= 0,
\]
Therefore $\mu$ is a minimizer to problem $\mathcal{FEA}$ because the objective function of this problem is always non-negative. 

\subsection{Technical Properties of Bellman Operators}
The Bellman operator $\mathbf T$ has the following properties.
\begin{lemma}\label{prop:basic_prop}
The Bellman operator $\mathbf T[V]$ has the following properties:
\begin{itemize}
\item (Monotonicity) If $V_1(x)\geq V_2(x)$, for any $x\in\mathbf{X}$, then $\mathbf T[V_1](x)\geq \mathbf T[V_2](x)$.
\item (Translational Invariant) For any constant $K\in\mathbb R$, $\mathbf T[V](x)-|K|\leq\mathbf T[V+K](x)\leq\mathbf T[V](x)+|K|$, for any $x\in\mathbf X'$.
\item (Contraction) There exists a positive vector $\{\xi(x)\}_{x\in\mathbf X}$ and a constant $\beta\in(0,1)$ such that $\|\mathbf T[V_1]-\mathbf T[V_2]\|_{\xi}\leq  \beta\|V_1-V_2\|_{\xi}$\footnote{$\|f\|_{\xi}=\max_{x\in{\mathbf{X}}} |f(x)|/\xi(x)$}.
\end{itemize}
\end{lemma}
\begin{proof}
The proof of monotonicity and constant shift properties follow directly from the definition of Bellman operator. Now we prove the contraction property. Recall that the $t-$element in state $x=(t,z,\omega)$ is a time counter, its transition probability is given by $\mathbf 1\{t'=t+1\}$ if $t<T-1$ and $\mathbf 1\{t'=t\}$ if $t=T-1$. Obviously the transition probability $\mathbb{P}(x^\prime|x,u)$, which is a multivariate probability distribution of state $x$, is less than or equal to the marginal probability distribution of $t-$element. Thus for vector $\{\xi(x)\}_{x\in\mathbf X}$ such that 
\begin{equation}\label{eq:xi}
\xi(x)=T-t\geq 0,\,\,\forall x\in\mathbf X, 
\end{equation}
we have that 
\[
\sum_{x'\in\mathbf X'}\xi(x')\mathbb{P}(x^\prime|x,u)\leq\sum_{x'\in\mathbf X'}\xi(x')\mathbf 1\{t'=t+1\}\leq \frac{T-1}{T}\xi(x),\,\,\forall x\in\mathbf X,\,\,\forall u\in\mathbf U(x).
\]
Here one observes that the effective ``discounting factor" is given by 
\begin{equation}\label{eq:beta}
\beta=\frac{T-1}{T}\in(0,1).
\end{equation}
Then for any vectors $V_1$, $V_2:\mathbf X\rightarrow\mathbb R$, 
\[
\begin{split}
\left|\mathbf T[V_1](x)-\mathbf T[V_2](x)\right|\leq&\max_{u\in\mathbf U}\left|\Pi_{\mathcal B(x)}\left(\mathbf{D}(x,u)+\sum_{x'\in\mathbf X'}\mathbb{P}(x^\prime|x,u)V_1(x')\right)-\Pi_{\mathcal B(x)}\left(\mathbf{D}(x,u)+\sum_{x'\in\mathbf X'}\mathbb{P}(x^\prime|x,u)V_2(x')\right) \right|\\
\leq &\max_{u\in\mathbf U}\left|\max\left\{\mathcal B(x),\sum_{x'\in\mathbf X'}\mathbb{P}(x^\prime|x,u)V_1(x')-\sum_{x'\in\mathbf X'}\mathbb{P}(x^\prime|x,u)V_2(x')\right\}\right|\\
\leq &\max_{u\in\mathbf U}\sum_{x'\in\mathbf X'}\mathbb{P}(x^\prime|x,u)|V_1(x')-V_2(x')|\\
\leq &\max_{x'\in\mathbf X'}\frac{|V_1(x')-V_2(x')|}{\xi(x')}\max_{u\in\mathbf U}\sum_{x'\in\mathbf X'}\xi(x')\mathbb{P}(x^\prime|x,u)\\
\leq &\max_{x'\in\mathbf X'}\frac{|V_1(x')-V_2(x')|}{\xi(x')}\beta\xi(x).
\end{split}
\]
This further implies that the following contraction property holds: $\|\mathbf T[V_1]-\mathbf T[V_2]\|_{\xi}\leq  \beta\|V_1-V_2\|_{\xi}$.
\end{proof}
Similarly, the Bellman operator $\mathbf T_R$ also has the following properties.
\begin{lemma}\label{prop:basic_prop_TR}
The Bellman operator $\mathbf T_R[V]$ is monotonic, translational invariant and it is a contraction mapping with respect to the $\|\cdot\|_\xi$ norm. 
\end{lemma}
The proof of this lemma is identical to the proof of Lemma \ref{prop:basic_prop} and is omitted for the sake of brevity.

\subsection{Proof of Theorem \ref{thm:opt}}
The first part of the proof is to show by induction that  for $x\in\mathbf X$,
\begin{equation}\label{induction_Vn}
V_N(x):=\mathbf T^N[V_{0}](x)=\min_{\mu}\Pi_{\mathcal B(x)}\left(\mathbb E\left[\sum_{t=0}^{N-1}\mathbf{D}(x_{t},u_t)+V_0(x_N)\mid x,\mu\right]\right).
\end{equation}
 For $N=1$, the definition of Bellman operator $\mathbf T$ implies that 
 \[
 V_1(x)=\mathbf T[V_{0}](x)=\min_{u\in\mathbf U(x)}\Pi_{\mathcal B(x)}(\mathbf{D}(x,u)+\mathbb E\left[  V_0(x')\mid x,\,u\right]).
 \]
By the induction hypothesis, assume \eqref{induction_Vn} holds at $N=k$. For $N=k+1$, 
\[
\begin{split}
&V_{k+1}(x):=\mathbf T^{k+1}[V_{0}](x)=\mathbf T[V_k](x)\\
=& \min_{u\in\mathbf U(x)}\Pi_{\mathcal B(x)}\left({\mathbf{D}}(x,u)+\sum_{x'\in\mathbf X'}\mathbb P(x'|x,u)\left[\Pi_{\mathcal B(x')}\left(\min_{\mu}\mathbb E\left[\sum_{t=0}^{k-1} {\mathbf{D}}(x_t,u_t)+V_0(x_{k})\mid x',\mu\right]\right)\right]\right)\\
=& \min_{u\in\mathbf U(x)}\max\Bigg\{\mathcal B(x),{\mathbf{D}}(x,u)+\sum_{x'\in\mathbf X'}\mathbb P(x'|x,u)\left[\max\Bigg\{-\infty,\min_{\mu}\mathbb E\left[\sum_{t=0}^{k-1} {\mathbf{D}}(x_t,u_t)+V_0(x_{k})\mid x',\mu\right]\Bigg\}\right]\Bigg\}\\
=&\min_{u\in\mathbf U(x)}\max\Bigg\{\mathcal B(x),{\mathbf{D}}(x,u)+\sum_{x'\in\mathbf X'}\mathbb P(x'|x,u)\left[\min_{\mu}\mathbb E\left[\sum_{t=0}^{k-1} {\mathbf{D}}(x_t,u_t)+V_0(x_{k})\mid x',\mu\right]\right]\Bigg\}\\
=&\min_{u\in\mathbf U(x)}\Pi_{\mathcal B(x)}\left(\mathbf{D}(x,u)+\sum_{x'\in\mathbf X'}\mathbb P(x'|x,u)\left[\min_{\mu}\mathbb E\left[\sum_{t=1}^{k} {\mathbf{D}}(x_t,u_t)+V_0(x_{k+1})\mid x',\mu\right]\right]\right)\\
=&\min_{\mu}\Pi_{\mathcal B(x)}\left(\mathbb E\left[\sum_{t=0}^{k} {\mathbf{D}}(x_t,u_t)+V_0(x_{k+1})\mid x,\mu\right]\right).
\end{split}
\]
Thus, the equality in \eqref{induction_Vn} is proved by induction. 

The second part of the proof is to show that $V^*(x_0)
:=\lim_{N\rightarrow\infty} V_{N}(x_0)$ and equation \eqref{eq:equiv} holds. Since $V_0(x)$ is bounded for any $x\in{\mathbf{X}}$, the first argument implies that
\[
\begin{split}
 V^*(x_0)= & \min_{\mu}\max\!\left\{\!0,\lim_{N\rightarrow\infty}\!\mathbb E\! \left[\sum_{t=0}^{N-1} \!{\mathbf{D}}(x_t,u_t)\!+V_0(x_N)\mid x_0,\mu\right]\right\}\\
\geq&  \min_{\mu}\max\left\{0,\mathbb E \left[\sum_{t=0}^{\infty} {\mathbf{D}}(x_t,u_t)\!\mid\! x_0,\mu\right]\right\}-\lim_{N\rightarrow\infty} \max_{x\in \mathbf X'}\mathbb P[x_N=x\mid x_0,\mu]\|V_0\|_\infty \\
\geq&\min_{\mu} \max\left\{0,\mathbb E\left[\sum_{t=0}^{\infty}\mathbf{D}(x_{t},u_t)\mid x_0,\mu\right]\right\}- \epsilon\|V_0\|_\infty.
\end{split}
\]
The first inequality is due to 1) $V_0$ is bounded and 2) $\mathbf{D}(x_{t},u_t)=0$ when $x_{t}$ is in the absorbing set $\mathcal X$. The second inequality follows from the fact that $x_t$ enters the absorbing set $\mathcal X$ after $T$ steps.
By similar arguments, one can also show that
\[
V^*(x_0)\leq \min_{\mu} \max\left\{0,\mathbb E\left[\sum_{t=0}^{\infty}\mathbf{D}(x_{t},u_t)\mid x_0,\mu\right]\right\}+ \epsilon\|V_0\|_\infty.
\]
Therefore, by taking $\epsilon\rightarrow 0$, the proof is completed.

The third part of the proof is to show the uniqueness of fixed point solution. Starting at $V_0:{\mathbf{X}}\rightarrow\mathbb R$ one obtains from iteration $V_{k+1}(x)=\mathbf T[V_k](x)$  that
\[
V_{k+1}(x)=\min_{u\in\mathbf U(x)}\max\left\{\mathcal B(x),D(x,u)+\mathbb E\left[V_k(x')\mid x,u\right]\right\}.
\]
By taking the limit, and noting that $V^*(x)=\lim_{k\rightarrow \infty}V_{k+1}(x)=\mathbf T[\lim_{k\rightarrow \infty}V_{k}](x)=\mathbf T[V^*](x)$, which implies $V$ is a fixed point of the Bellman equation. Furthermore, the fixed point is unique because if there exists a different fixed point $\widetilde V$, then $\mathbf T^k[\widetilde V](x)=\widetilde V(x)$ for any $k\geq 0$. As $k\rightarrow\infty$, one obtains $\widetilde V(x)= V^*(x)$ which yields a contradiction.  

\subsection{Proof of Theorem \ref{thm:syn_Q}}
The convergence proof of two phase $Q-$learning is split into the following two steps.
\noindent\paragraph{Step 1 (Convergence of $ Q-$update)}
We first show the convergence of $ Q-$update (feasible set update) in two phase $Q-$learning. 
Recall that the state-action Bellman operator $\mathbf F$ is given as follows:
\[
\mathbf F[Q](x,u)=\max\left\{\mathcal B(x),\mathbf{D}(x,u)+\sum_{x'\in\mathbf X'}\mathbb P(x'|x,u) \min_{u'\in\mathbf U(x')}Q(x',u')\right\}.
\]
Therefore, the $ Q-$update can be re-written as
\[
 Q_{k+1}(x,u)= (1-\zeta_{2,k}(x,u))Q_{k}(x,u)+\zeta_{2,k}(x,u) \left(\Pi_{\mathcal B(x)}\left(\mathbf{D}(x,u)+ \sum_{x'\in\mathbf X'}\mathbb P(x'|x,u)\min_{u'\in\mathbf{U}(x')} Q_{k}(x',u') \right)+N_k(x,u)\right),
 \]
where the noise term is given by
\begin{equation}\label{eq:noise_term}
N_k(x,u)=\Pi_{\mathcal B(x)}\left(\mathbf{D}(x,u)+\frac{1}{N}\sum_{m=1}^{N}\min_{u^{\prime,m}\in\mathbf{U}(x^{\prime,m})} Q_{k}(x^{\prime,m},u^{\prime,m})\right)-\Pi_{\mathcal B(x)}\left(\mathbf{D}(x,u)+ \sum_{x'\in\mathbf X'}\mathbb P(x'|x,u)\min_{u'\in\mathbf{U}(x')} Q_{k}(x',u') \right),
\end{equation}
for which  $N_k(x,u)\rightarrow 0$ as $k\rightarrow \infty$ and for any $k\in\mathbb N$, 
\[
N^2_k(x,u)\leq \left|\frac{1}{N}\sum_{m=1}^{N}\min_{u^{\prime,m}\in\mathbf{U}(x^{\prime,m})} Q_{k}(x^{\prime,m},u^{\prime,m}) -\sum_{x'\in\mathbf X'}\mathbb P(x'|x,u)\min_{u'\in\mathbf{U}(x')} Q_{k}(x',u')\right|^2\leq 2\max_{x,u}Q^2_k(x,u).
\]
Then the assumptions in Proposition 4.5 in \cite{bertsekas1995neuro} on the noise term $N_k(x,u)$ are verified. Furthermore, following the same analysis from 
 Proposition \ref{prop:basic_prop} that $\mathbf{T}$ is a contraction operator with respect to the $\xi$ norm, for any two state-action value functions $Q_1(x,u)$ and $Q_2(x,u)$, we have that
\begin{equation}\label{eq:contraction_Q}
\begin{split}
&\left|\Pi_{\mathcal B(x)}\left(\mathbf{D}(x,u)+ \sum_{x'\in\mathbf X'}\mathbb P(x'|x,u)\min_{u'\in\mathbf{U}(x')} Q_{1}(x',u') \right)-\Pi_{\mathcal B(x)}\left(\mathbf{D}(x,u)+ \sum_{x'\in\mathbf X'}\mathbb P(x'|x,u)\min_{u'\in\mathbf{U}(x')} Q_{2}(x',u') \right)\right|\\
\leq &\left| \sum_{x'\in\mathbf X'}\mathbb P(x'|x,u)\min_{u'\in\mathbf{U}(x')} Q_{1}(x',u') - \sum_{x'\in\mathbf X'}\mathbb P(x'|x,u)\min_{u'\in\mathbf{U}(x')} Q_{2}(x',u')\right|\\
\leq & \sum_{x'\in\mathbf X'}\mathbb P(x'|x,u)\max_{u'\in\mathbf{U}(x')} \left|Q_{1}(x',u')- Q_{2}(x',u')\right| \\
\leq & \sum_{x'\in\mathbf X'}\mathbb P(x'|x,u)\xi(x')\max_{x'\in\mathcal X}\max_{u'\in\mathbf{U}(x')} \frac{\left|Q_{1}(x',u')- Q_{2}(x',u')\right|}{\xi(x')}\leq \beta\xi(x) \left\|  Q_{1}- Q_{2}\right\|_\xi.
\end{split}
\end{equation}
Here $\|Q\|_\xi=\max_{x'\in\mathbf X}\max_{u'\in\mathbf{U}(x')} {\left|Q(x',u')\right|}/{\xi(x')}$ and $\beta\in(0,1)$ is given by \eqref{eq:beta} and $\xi$ is given by \eqref{eq:xi}.
The first inequality is due to the fact that projection operator $\Pi_{\mathcal B(x)}$ is non-expansive. The second inequality follows from triangular inequality and 
\[
\sum_{x'\in\mathbf X'}\mathbb P(x'|x,u)\left| \min_{u'\in\mathbf{U}(x')} Q_{1}(x',u')-\min_{u'\in\mathbf{U}(x')} Q_{2}(x',u')\right|\leq\sum_{x'\in\mathbf X'}\mathbb P(x'|x,u)\max_{u'\in\mathbf{U}(x')} \left|Q_{1}(x',u')- Q_{2}(x',u')\right|.
\]
The third inequality holds, due to the fact $\sum_{x'\in\mathbf X'}\mathbb P(x^\prime|x,u)\xi(x')\leq \beta\xi(x)$ for $\beta\in(0,1)$. Therefore the above expression implies that $\|\mathbf F[Q_1]-\mathbf F[Q_2]\|_\xi\leq \beta\left\|  Q_{1}- Q_{2}\right\|_\xi$ for some $\beta\in(0,1)$, i.e., $\mathbf F$ is a contraction mapping with respect to the $\xi$ norm. 

By combining these arguments, all assumptions in Proposition 4.5 in \cite{bertsekas1995neuro} are justified. This in turns implies the convergence of $\{Q_k(x,u)\}_{k\in\mathbb N}$ to $Q^*(x,u)$ component-wise, where $Q^*$ is the unique fixed point solution of $\mathbf F[Q](x,u)=Q(x,u)$.

\noindent\paragraph{Step 2 (Convergence of $H-$update)}
Now we show the convergence of $H-$update (objective function update) in two phase $Q-$learning. Since $ Q$ converges at a faster timescale than $H$, the $H-$update can be rewritten using the converged quantity, i.e., $ Q^*$, as follows:
\[
H_{k+1}(x,u)=H_{k}(x,u)+\zeta_{1,k}(x,u)\cdot \left( \mathbf{R}(x,u)+ \frac{1}{N}\sum_{m=1}^{N}\min_{u^{\prime,m}\in \mathbf U_{\text{feas}}(Q^*,x^{\prime,m})}H_{k}(x^{\prime,m},u^{\prime,m}) - H_{k}(x,u)\right)
\]
Recall that the state-action Bellman operator $\mathbf F_R$ is given as follows:
\[
\mathbf F_R[H](x,u)= \mathbf{R}(x,u)+\sum_{x^\prime\in \mathbf X'}\mathbb{P}(x^\prime|x,u)\min_{u'\in\mathbf U_{\text{feas}}(Q^*,x')}H\left(x^\prime,u^\prime\right).
\]
Therefore, the $ H-$update can be re-written as the following form:
\[
\begin{split}
H_{k+1}(x,u)=& (1-\zeta_{1,k}(x,u))H_{k}(x,u)\\
 &+\zeta_{1,k}(x,u) \left(\mathbf{R}(x,u)+\sum_{x^\prime\in \mathbf X'}\mathbb{P}(x^\prime|x,u)\min_{u'\in\mathbf U_{\text{feas}}(Q^*,x')}H_k\left(x^\prime,u^\prime\right)+\mathcal{N}_k(x,u)\right),
 \end{split}
 \]
where the noise term is given by
\begin{equation}\label{eq:noise_term2}
\mathcal{N}_k(x,u)=\frac{1}{N}\sum_{m=1}^{N}\min_{u^{\prime,m}\in \mathbf U_{\text{feas}}(Q^*,x^{\prime,m})}H_{k}(x^{\prime,m},u^{\prime,m}) -\sum_{x^\prime\in \mathbf X'}\mathbb{P}(x^\prime|x,u)\min_{u'\in\mathbf U_{\text{feas}}(Q^*,x')}H_k\left(x^\prime,u^\prime\right),
\end{equation}
such that $\mathbb E[\mathcal{N}_k(x,u)\mid\mathcal F_k]=0$ and for any $k\in\mathbb N$, 
\[
\begin{split}
\mathcal{N}^2_k(x,u)\leq & \left|\frac{1}{N}\sum_{m=1}^{N}\min_{u^{\prime,m}\in \mathbf U_{\text{feas}}(Q^*,x^{\prime,m})}H_{k}(x^{\prime,m},u^{\prime,m}) -\sum_{x^\prime\in \mathbf X'}\mathbb{P}(x^\prime|x,u)\min_{u'\in\mathbf U_{\text{feas}}(Q^*,x')}H_k\left(x^\prime,u^\prime\right)\right|^2\\
\leq & 2\max_{x,u}Q^2_k(x,u).
\end{split}
\]
Then the assumptions in Proposition 4.4 in \cite{bertsekas1995neuro} on the noise term $\mathcal{N}_k(x,u)$ are verified. Following the analogous arguments in \eqref{eq:contraction_Q}, we can also show that $\|\mathbf F_R[H_1]-\mathbf F_R[H_2]\|_\xi\leq \beta\left\| H_{1}- H_{2}\right\|_\xi$ where $\beta\in(0,1)$ is given by \eqref{eq:beta} and $\xi$ is given by \eqref{eq:xi}, i.e., $\mathbf F_R$ is a contraction mapping with respect to the $\xi$ norm. By combining these arguments, all assumptions in Proposition 4.4 in \cite{bertsekas1995neuro} are justified. This in turns implies the convergence of $\{H_k(x,u)\}_{k\in\mathbb N}$ to $H^*(x,u)$ component-wise, where $Q^*$ is the unique fixed point solution of $\mathbf F_R[H](x,u)=H(x,u)$.

\subsection{Proof of Theorem \ref{thm:asy_Q}}
The convergence proof of asynchronous two phase $Q-$learning is split into the following two steps.

\noindent\paragraph{Step 1 (Convergence of $ Q-$update)}
Similar to the proof of Theorem \ref{thm:syn_Q}, the $Q-$update in asynchronous two phase $Q-$learning can be written as:
\[
Q_{k+1}(x,u)= (1-\zeta_{2,k}(x,u))Q_{k}(x,u) +\zeta_{2,k}(x,u)(\Theta_k(x,u)+\Psi_k(x,u)),
\]
where
\[
 \Theta_k(x,u)=\left\{\begin{array}{ll}
  \Pi_{\mathcal B(x)}\bigg(\mathbf{D}(x,u)+ \sum_{x'\in\mathbf X'}\mathbb P(x'|x,u)\min_{u'\in\mathbf{U}(x')} Q_{k}(x',u') \bigg)&\text{if $(x,u)=(x_k,u_k)$}\\
 Q_{k}(x,u)&\text{otherwise}
 \end{array}\right.
 \]
 and the noise term is given by
 \[
 \Psi_k(x,u)=\left\{\begin{array}{ll}
N_k(x,u) &\text{if $(x,u)=(x_k,u_k)$}\\
0&\text{otherwise}
 \end{array}\right.
 \]
 with $N_k$ defined in \eqref{eq:noise_term}.
 Since $N_k(x,u)\rightarrow 0$ as $k\rightarrow \infty$, it can also be seen that  $\Psi_k(x,u)\rightarrow 0$ as $k\rightarrow \infty$. Furthermore, for any $k\in\mathbb N$, we also have that $\Psi^2_k(x,u)\leq N^2_k(x,u)\leq 2\max_{x,u}Q^2_k(x,u)$. Then the assumptions in Proposition 4.5 in \cite{bertsekas1995neuro} on the noise term $N_k(x,u)$ are verified. Now we define the asynchronous Bellman operator
\[
\widetilde{\mathbf F}[ Q](x,u)=\left\{\begin{array}{ll}
  \Pi_{\mathcal B(x)}\bigg(\mathbf{D}(x,u)+ \sum_{x'\in\mathbf X'}\mathbb P(x'|x,u)\min_{u'\in\mathbf{U}(x')} Q(x',u') \bigg)&\text{if $(x,u)=(x_k,u_k)$}\\
 Q(x,u)&\text{otherwise}
 \end{array}\right..
 \]
It can easily checked that the fixed point solution of $\mathbf F[ Q](x,u)= Q(x,u)$, i.e., $Q^*$, is also a fixed point solution of $\widetilde{\mathbf F}[ Q](x,u)= Q(x,u)$. Next we want to show that $\widetilde{\mathbf F}[ Q]$ is a contraction operator with respect to $\xi$. Let $\{\ell_{k}\}$ be a strictly increasing sequence ($\ell_{k}<\ell_{k+1}$ for all $k$) such that $\ell_0=0$, and every state-action pair $(x,u)$ in $\mathbf X\times\mathbf U$ is being updated at least once during this time period. Since every state action pair is visited infinitely often, Borel-Cantelli lemma \cite{ross1996stochastic} implies that for each finite $k$, both $\ell_{k}$ and $\ell_{k+1}$ are finite. For any $\ell\in[\ell_{k},\ell_{k+1}]$,  the result in \eqref{eq:contraction_Q} implies the following expression:
\[
\begin{array}{ll}
|\widetilde{\mathbf F}^{\ell+1}[ Q](x,u)-Q^*(x,u)|\leq \beta\xi(x) \left\|  \widetilde{\mathbf F}^{\ell}[ Q]- Q^*\right\|_\xi &\text{if $(x,u)=(x_k,u_k)$}\\
|\widetilde{\mathbf F}^{\ell+1}[ Q](x,u)-Q^*(x,u)|= |\widetilde{\mathbf F}^{\ell}[ Q](x,u)-Q^*(x,u)|&\text{otherwise}
\end{array}
\] 
From this result, one can first conclude that $\widetilde{\mathbf F}[ Q]$ is a non-expansive operator, i.e.,
\[
|\widetilde{\mathbf F}^{\ell+1}[ Q](x,u)-Q^*(x,u)|\leq \xi(x) \left\|  \widetilde{\mathbf F}^{\ell}[ Q]- Q^*\right\|_\xi.
\] 
Let $l(x,u)$ be the last index strictly between $\ell_{k}$ and $\ell_{k+1}$ where the state-action pair $(x,u)$ is updated. There exists $\beta\in(0,1)$ such that
\[
|\widetilde{\mathbf F}^{\ell_{k+1}}[ Q](x,u)-Q^*(x,u)|\leq \beta\xi(x) \left\|  \widetilde{\mathbf F}^{l(x,u)}[ Q]- Q^*\right\|_\xi
\]
From the definition of $\ell_{k+1}$, it is obvious that  $\ell_{k}<\max_{x,u}l(x,u)< \ell_{k+1}$. The non-expansive property of $\widetilde{\mathbf F}$ also implies that $\left\|  \widetilde{\mathbf F}^{l(x,u)}[ Q]- Q^*\right\|_\xi\leq \left\|  \widetilde{\mathbf F}^{\ell_{k}}[ Q]- Q^*\right\|_\xi$. Therefore we have that 
\[
|\widetilde{\mathbf F}^{\ell_{k+1}}[ Q](x,u)-Q^*(x,u)|\leq \beta\xi(x) \left\|  \widetilde{\mathbf F}^{\ell_{k}}[ Q]- Q^*\right\|_\xi.
\]
Combining these arguments implies that $\|\widetilde{\mathbf F}^{\ell_{k+1}}[ Q]-Q^*\|_\xi\leq \beta \left\|  \widetilde{\mathbf F}^{\ell_{k}}[ Q]- Q^*\right\|_\xi$. Thus for $\delta_k=\ell_{k+1}-\ell_{k}>1$ and $Q_k(x,u)=\widetilde{\mathbf F}^{\ell_{k}}[ Q](x,u)$, the following contraction property holds:
\begin{equation}\label{eq:contraction_Q_asy}
\|\widetilde{\mathbf F}^{\delta_k}[ Q_k]-Q^*\|_\xi\leq \beta \left\| Q_k- Q^*\right\|_\xi,
\end{equation}
where the following fixed point equation holds: $\widetilde{\mathbf F}^{\delta_k}[Q^*] (x,u)=Q^* (x,u)$.
Then by Proposition 4.5 in \cite{bertsekas1995neuro}, the sequence $\{Q_k(x,u)\}_{k\in\mathbb N}$ converges to $Q^*(x,u)$ component-wise, where $Q^*$ is the unique fixed point solution of both $\mathbf F[Q](x,u)=Q(x,u)$ and $\widetilde{\mathbf F}[ Q](x,u)= Q(x,u)$.

\noindent\paragraph{Step 2 (Convergence of $H-$update)}
Since $ Q$ converges at a faster timescale than $H$, the $H-$update in asynchronous two phase $Q-$learning can be rewritten using the converged quantity, i.e., $ Q^*$, as follows:
\[
H_{k+1}(x,u)= (1-\zeta_{1,k}(x,u))H_{k}(x,u) +\zeta_{1,k}(x,u)(\Lambda_k(x,u)+\Phi_k(x,u)),
\]
where
\[
 \Lambda_k(x,u)=\left\{\begin{array}{ll}
  \mathbf{R}(x,u)+\sum_{x^\prime\in \mathbf X'}\mathbb{P}(x^\prime|x,u)\min_{u'\in\mathbf U_{\text{feas}}(Q^*,x')}H_k\left(x^\prime,u^\prime\right)&\text{if $(x,u)=(x_k,u_k)$}\\
 H_{k}(x,u)&\text{otherwise}
 \end{array}\right.
 \]
 and the noise term is given by
 \[
 \Phi_k(x,u)=\left\{\begin{array}{ll}
\mathcal{N}_k(x,u) &\text{if $(x,u)=(x_k,u_k)$}\\
0&\text{otherwise}
 \end{array}\right.
 \]
 with $\mathcal{N}_k$ defined in \eqref{eq:noise_term2}. Since $\mathbb E[\mathcal{N}_k(x,u)\mid\mathcal F_k]=0$, we have that $\mathbb E[\Phi_k(x,u)\mid\mathcal F_k]=0$, i.e., $ \Phi_k(x,u)$ is a Martingale difference. Furthermore we have that $\Phi^2_k(x,u)\leq \mathcal{N}^2_k(x,u)\leq 2\max_{x,u}Q^2_k(x,u)$  for $k\in\mathbb N$. The above arguments verify the assumptions in Proposition 4.4 in \cite{bertsekas1995neuro} on the noise term $\Phi_k(x,u)$. Now define the asynchronous Bellman operator
\[
\widetilde{\mathbf F}_R[H](x,u)=\left\{\begin{array}{ll}
  \mathbf{R}(x,u)+\sum_{x^\prime\in \mathbf X'}\mathbb{P}(x^\prime|x,u)\min_{u'\in\mathbf U_{\text{feas}}(Q^*,x')}H\left(x^\prime,u^\prime\right)&\text{if $(x,u)=(x_k,u_k)$}\\
 H(x,u)&\text{otherwise}
 \end{array}\right..
 \]
It can easily checked that the fixed point solution of $\mathbf F_R[H](x,u)=H(x,u)$, i.e., $H^*$, is also a fixed point solution of $\widetilde{\mathbf F}_R[H](x,u)=H(x,u)$.
Then following analogous arguments from step 1 (in particular expression \eqref{eq:contraction_Q_asy}), for $\delta_k=\ell_{k+1}-\ell_{k}>1$ and $H_k(x,u)=\widetilde{\mathbf F}^{\ell_{k}}[ H](x,u)$, one shows that $\|\widetilde{\mathbf F}^{\delta_k}[ H_k]-H^*\|_\xi\leq \beta \left\| H_k- H^*\right\|_\xi$ for some $\beta\in(0,1)$, which further implies the following fixed point equation holds: $\widetilde{\mathbf F}^{\delta_k}[H^*] (x,u)=H^* (x,u)$. Thus by Proposition 4.4 in \cite{bertsekas1995neuro}, the sequence $\{H_k(x,u)\}_{k\in\mathbb N}$ converges to $H^*(x,u)$ component-wise, where $Q^*$ is the unique fixed point solution of both $\widetilde{\mathbf F}_R[H](x,u)=H(x,u)$ and $\mathbf F_R[H](x,u)=H(x,u)$.

\end{document}